\documentclass{colt2014}
\usepackage{times}

\newcommand{\ourtitle}{A Second-order Bound with Excess Losses}
\title{\ourtitle}

\coltauthor{\Name{Pierre Gaillard} \Email{pierre@gaillard.me}\\
\addr EDF R\&D, Clamart, France
\AND
\Name{Gilles Stoltz} \Email{stoltz@hec.fr}\\
\addr HEC Paris, CNRS, Jouy-en-Josas, France
\AND
\Name{Tim {van Erven}} \Email{tim@timvanerven.nl}\\
\addr Universit\'e Paris-Sud, Orsay, France
}

\usepackage[latin1]{inputenc}
\usepackage{algorithm}
\usepackage{cancel}
\usepackage{dsfont}
\allowdisplaybreaks

\newcommand{\R}{\mathbb{R}}

\newcommand{\indic}[1]{\mathds{1}_{#1}}

\renewcommand{\leq}{\leqslant}
\renewcommand{\geq}{\geqslant}
\renewcommand{\epsilon}{\varepsilon}
\renewcommand{\hat}{\widehat}
\newcommand{\argmin}{\mathop{\mathrm{argmin}}}
\newcommand{\cF}{\mathcal{F}}

\newcommand{\loss}{\ell}

\newcommand{\algloss}{\hat{\ell}}

\newcommand{\bq}{\boldsymbol{q}}
\renewcommand{\b}{\boldsymbol}
\newcommand{\bp}{\boldsymbol{p}}

\newcommand{\cregret}{R^\textrm{c}} 
\newcommand{\Var}{\mathrm{Var}}
\renewcommand{\P}{\mathds{P}}
\newcommand{\E}{\mathds{E}}

\newcommand{\e}{\mathrm{e}}
\renewcommand{\d}{\,\mathrm{d}}
\newcommand{\norm}[1]{\left\|#1\right\|}
\newcommand{\diag}{\mathrm{diag}}
\renewcommand{\indent}{\hspace*{0.5cm}}
\renewcommand{\tilde}{\widetilde}
\newcommand{\lt}{\tilde{\ell}}
\newcommand{\pt}{\tilde{p}}
\newcommand{\bpt}{\tilde{\bp}}

\newcommand{\cA}{\mathcal{A}}
\newcommand{\half}{\frac{1}{2}}

\newtheorem{ass}{Assumption}

\begin{document}
\maketitle

\begin{abstract}
We study online aggregation of the predictions of experts, and first
show new second-order regret bounds in the standard setting, which
are obtained via a version of the Prod algorithm (and also a version of
the polynomially weighted average algorithm) with multiple learning
rates. These bounds are in terms of excess losses, the differences
between the instantaneous losses suffered by the algorithm and the ones
of a given expert. We then demonstrate the interest of these bounds in
the context of experts that report their confidences as a number in the
interval $[0,1]$ using a generic reduction to the standard setting. We
conclude by two other applications in the standard setting, which
improve the known bounds in case of small excess losses and show a
bounded regret against i.i.d.\ sequences of losses.
\end{abstract}

\section{Introduction}
\label{sec:introduction}

In the (simplest) setting of prediction with expert advice,
a learner has to make online sequential predictions over a series of
rounds, with the help of $K$ experts
\citep{FreundSchapire1997,LittlestoneWarmuth1994,Vovk1998,CesaBianchiLugosi2006}.
In each round $t = 1,\ldots, T$,
the learner makes a prediction by choosing a vector $\bp_t =
(p_{1,t},\ldots,p_{K,t})$ of nonnegative weights that sum to one. Then
every expert $k$ incurs a loss $\ell_{k,t} \in [a,b]$ and the learner's
loss is $\algloss_t = \bp_t^\top \b \ell_t = \sum_{k=1}^K p_{k,t}
\ell_{k,t}$, where $\b \ell_t = (\ell_{1,t},\ldots,\ell_{K,t})$.
The goal of the learner is to control his cumulative loss, which
he can do by controlling his regret $R_{k,T}$ against each expert $k$, where
$R_{k,T} = \sum_{t \leq T} \big(\algloss_t - \loss_{k,t}\big)$.
In the worst case, the best bound on the standard regret $R_{k,T}$ that
can be guaranteed is of order $O\big(\sqrt{T \ln K}\big)$; see, e.g.,
\citet{CesaBianchiLugosi2006}, but this can be improved. For
example, when losses take values in $[0,1]$,
$R_{k,T} = O\big(\sqrt{L_{k,T} \ln K}\big)$, with $L_{k,T} =
\sum_{t=1}^T \loss_{k,t}$, is also possible, which is better when the
losses are small---hence the name \emph{improvement for small losses} for
this type of bounds \citep{CesaBianchiLugosi2006}.

\paragraph{Second-order bounds}
\citet{CesaBianchiMansourStoltz2007} raised
the question of whether it was possible to improve even further by
proving second-order (variance-like) bounds on the regret. They
could establish two types of bound, each with its own advantages.
The first is of the form
\begin{equation}\label{eq:TunedProd2007}
R_{k,t} \leq
\frac{\ln K}{\eta} + \eta \sum_{t=1}^T \ell_{k,t}^2
\end{equation}
for all experts $k$, where $\eta \leq 1/2$ is a parameter of the
algorithm. If one could optimize $\eta$ with hindsight knowledge of the
losses, this would lead to the desired
bound
\begin{equation}\label{eq:impossibletuning}
\cancel{R_{k,T} = O\left({\sqrt{\ln K \,\, \textstyle{\sum_{t=1}^T
\ell_{k,t}^2}}}\right)},
\end{equation}
but, unfortunately, no method is known that actually achieves
\eqref{eq:impossibletuning} for all experts $k$ simultaneously without
such hindsight knowledge. As explained by
\citet{CesaBianchiMansourStoltz2007} and \citet{HazanKale2010}, the
technical difficulty is that the optimal $\eta$ would depend on $\sum_t
\ell_{k^\star_T,t}^2$, where
\[
k^\star_T \in \mathop{\mathrm{argmin}}_{k = 1,\ldots,K} \left\{
\sum_{t=1}^T \ell_{k,t} + \sqrt{\ln K \,\, \sum_{t=1}^T \ell_{k,t}^2}
\right\}.
\]
But, because $k^\star_T$ can vary with $T$, the sequence of the $\sum \ell_{k^\star_t,t}^2$
is not monotonic and, as a consequence, standard tuning
methods (like for example the doubling trick) cannot be applied.

This is why this issue --- when
hindsight bounds seem too good to be obtained in a sequential fashion
--- is sometimes referred to as the problem of \emph{impossible tunings}. Improved bounds with
respect to~\eqref{eq:TunedProd2007} have been obtained by
\citet{HazanKale2010} and \citet{ChiangEtAl2012} but they suffer from
the same impossible tuning issue. \medskip

The second type of bound distinguished by
\citet{CesaBianchiMansourStoltz2007} is of the form
\begin{equation}\label{eq:ExponentialWeightsVariances}
  R_{k,T} = O\left( { \sqrt{\ln K \,\, \textstyle{\sum_{t=1}^T}
  v_{t}}}\right),
\end{equation}
uniformly over all experts $k$, where $v_{t} = \sum_{k \leq K} p_{k,t} \big(\hat
\ell_t- \ell_{k,t}\big)^2$ is the variance of the losses at instance $t$
under distribution $\bp_t$. It can be achieved by a variant of the
exponentially weighted average forecaster using the appropriate tuning
of a time-varying learning rate $\eta_t$
\citep{CesaBianchiMansourStoltz2007,DeRooijEtAl2013}. The
bound~\eqref{eq:ExponentialWeightsVariances} was shown in the mentioned
references to have several interesting consequences (see Section~\ref{sec:OtherApplications}).
Its main drawback comes from its uniformity: it does not reflect that it is harder to compete
with some experts than with other ones.

\paragraph{Excess losses}

Instead of uniform regret bounds like
\eqref{eq:ExponentialWeightsVariances}, we aim to get expert-dependent
regret bounds. The key quantities in our analysis turn out to be the
instantaneous \emph{excess losses} $\loss_{k,t} - \algloss_t$, and we
provide in Sections~\ref{sec:MLProd} and~\ref{sec:TuningEta} a new
second-order bound of the form
\begin{equation}\label{eq:NewBound}
R_{k,T} = O\left(\sqrt{\ln K \,\, \textstyle{\sum_{t=1}^T} (\algloss_t- \loss_{k,t})^2}\right),
\end{equation}
which holds for all experts $k$ simultaneously. To achieve this bound, we develop a
variant of the Prod algorithm of \citet{CesaBianchiMansourStoltz2007}
with two innovations: first we extend the analysis for Prod to multiple
learning rates $\eta_k$ (one for each expert) in the spirit of a variant
of the Hedge algorithm with multiple learning rates proposed by
\citet{BlumMansour2007}. Standard tuning techniques of the learning rates would then
still lead to an additional $O(\sqrt{K \ln T})$ multiplicative factor,
so, secondly, we develop new techniques that bring this factor down to
$O(\ln \ln T)$, which we consider to be essentially a constant.

The interest of the bound~\eqref{eq:NewBound} is demonstrated in Sections~\ref{sec:appliconf}
and~\ref{sec:OtherApplications}.
Section~\ref{sec:appliconf} considers the setting of prediction with experts
that report their confidences as a number in the interval $[0,1]$, which was first studied by
\citet{BlumMansour2007}. Our general bound~\eqref{eq:NewBound}
leads to the first bound on the confidence regret that scales optimally
with the confidences of each expert.
Section~\ref{sec:OtherApplications} returns to the standard setting described at the beginning
of this paper: we show an improvement for small excess losses, which
supersedes the basic improvement
for small losses described at the beginning of the introduction. Also, we prove that in the special case of
independent, identically distributed losses, our bound leads to a
constant regret.

\section{A new regret bound in the standard setting}
\label{sec:MLProd}

We extend the Prod algorithm of \citet{CesaBianchiMansourStoltz2007}
to work with multiple learning rates.

\begin{algorithm}
\caption{Prod with multiple learning rates (ML-Prod)}
\label{algo:prod}
	\smallskip
    	\emph{Parameters}: a vector $\b\eta = (\eta_1,\ldots,\eta_K)$ of learning rates \\
	\emph{Initialization}: a vector $\b w_0 =
        (w_{1,0},\ldots,w_{K,0})$ of nonnegative weights that sum to $1$

        \smallskip
	\emph{For} each round $t = 1,\,2,\,\ldots$ \\
	\indent 1. form the mixture $\bp_t$ defined component-wise by
	$\displaystyle{p_{k,t} = {\eta_{k} w_{k,t-1}} \big/ {\b\eta^\top} {\b w_{t-1}}}$ \\
	\indent 2. observe the loss vector $\b \ell_t$ and incur loss $\algloss_t = \b p_t^\top \b \ell_t$\\
	\indent 3. for each expert $k$ perform the update
		$\displaystyle{w_{k,t} = w_{k,t-1} \bigl( 1+\eta_k \bigl(\algloss_t - \ell_{k,t} \bigr) \bigr)}$ \\
\end{algorithm}

\begin{theorem}
\label{th:lmprod-plain}
For all sequences of loss vectors $\b \ell_t \in [0,1]^K$, the cumulative loss of Algorithm~\ref{algo:prod}
run with learning rates $\eta_k \leq 1/2$ is bounded by
\[
\sum_{t=1}^T \hat{\ell}_{t} \leq \min_{1 \leq k \leq K} \left\{ \sum_{t=1}^T \ell_{k,t} + \frac{1}{\eta_k} \ln \frac{1}{w_{k,0}}
+ \eta_k \sum_{t=1}^T \bigl(\algloss_t - \ell_{k,t} \bigr)^2 \right\}.
\]
\end{theorem}

If we could optimize the bound of the theorem
with respect to $\eta_k$, we would obtain the desired result:
\begin{equation}
\label{eq:optiSOB}
\sum_{t=1}^T \algloss_t \leq \min_{1 \leq k \leq K} \left\{ \sum_{t=1}^T \ell_{k,t} +
2 \sqrt{\sum_{t=1}^T V_{k,t} \ln \frac{1}{w_{k,0}}} \right\}
\end{equation}
where  $V_{k,t} = \bigl(\algloss_t - \ell_{k,t} \bigr)^2$.
The question is therefore how to get the optimized bound~\eqref{eq:optiSOB}
in a fully sequential way. Working in regimes (resorting to some doubling trick) seems suboptimal,
since $K$ quantities $\sum_t V_{k,t}$ need to be controlled simultaneously and new regimes will start
as soon as one of these quantities is larger than some dyadic threshold. This would lead to an
additional $O(\sqrt{K \ln T})$ multiplicative factor in the bound. We propose in Section~\ref{sec:TuningEta}
a finer scheme, based on time-varying learning rates $\eta_{k,t}$, which only costs a
multiplicative $O(\ln \ln T)$ factor in the regret bounds. Though the analysis of a single time-varying parameter
is rather standard since the paper by \citet{AuerEtAl2002}, the
analysis of multiple such parameters is challenging and does not follow from a routine
calculation. That the ``impossible tuning'' issue does not arise here was quite surprising to us.

\paragraph{Empirical variance of the excess losses} A consequence of~\eqref{eq:optiSOB} is the following bound,
which is in terms of the empirical variance of the excess losses $\ell_{k,t} - \algloss_t$:
\begin{equation}
\label{eq:varexc}
\sum_{t=1}^T \algloss_t \leq \min_{1 \leq k \leq K} \left\{ \sum_{t=1}^T \ell_{k,t}
+ 4 \ln \frac{1}{w_{k,0}}
+ 2 \sqrt{\sum_{t=1}^T \left(\algloss_t - \ell_{k,t} - \frac{R_{k,T}}{T} \right)^2 \ln \frac{1}{w_{k,0}}} \right\}.
\end{equation}

\begin{proposition}
Suppose losses take values in $[0,1]$. If \eqref{eq:optiSOB} holds, then
\eqref{eq:varexc} holds.
\end{proposition}
%
\begin{proof}
A bias-variance decomposition indicates that, for each $k$,
\begin{equation}\label{eq:biasvariance}
\sum_{t=1}^T V_{k,t} = \sum_{t=1}^T \bigl( \algloss_t - \ell_{k,t} \bigr)^2 = \sum_{t=1}^T \bigl(\algloss_t - \ell_{k,t} - R_{k,T}/T\bigr)^2
+ T\,\bigl( {R_{k,T}}/{T} \bigr)^2.
\end{equation}
It is sufficient to prove the result when the minimum is restricted to
$k$ such that $R_{k,T} \geq 0$. For such $k$, \eqref{eq:optiSOB} implies
that $R_{k,T}^2 \leq 4 T \ln(1/w_{k,0})$. Substituting this into the
rightmost term of \eqref{eq:biasvariance}, the result
into~\eqref{eq:optiSOB}, and using that $\sqrt{x+y} \leq \sqrt{x} +
\sqrt{y}$ for $x,y \geq 0$ concludes the proof.
\end{proof}

\begin{proof}[\textbf{of Theorem~\ref{th:lmprod-plain}}]
The proof follows from a simple adaptation of Lemma~2 in \citet{CesaBianchiMansourStoltz2007}
and takes some inspiration from Section~6 of \citet{BlumMansour2007}.

For $t \geq 0$, we denote by $\b r_t \in [-1,1]^K$ the instantaneous regret vector defined component-wise by $r_{k,t} = \algloss_t - \ell_{k,t}$ and we define $W_t = \sum_{k=1}^K w_{k,t}$. We bound $\ln W_T$ from above and from below.

On the one hand, using the inequality $\ln(1+x) \geq x-x^2$ for all $x \geq -1/2$ (stated as Lemma~1 in~\citealp{CesaBianchiMansourStoltz2007}),
we have, for all experts $k$, that
\begin{equation*}
\ln W_T \geq \ln w_{k,T} =   \ln w_{k,0} + \sum_{t=1}^T \ln \bigl( 1+\eta_k r_{k,t} \bigr)
		\geq  \ln w_{k,0} + \eta_k \sum_{t=1}^T r_{k,t} - \eta_k^2 \sum_{t=1}^T r_{k,t}^2\,.
\end{equation*}
The last inequality holds because, by assumption, $\eta_k \leq 1/2$
and hence $\eta_k \bigl(\algloss_t - \ell_{k,t} \bigr) \leq 1/2$ as well.

We now show by induction that, on the other hand, $W_T = W_0 = 1$ and thus that $\ln W_T = 0$.
By definition of the weight update (step 3 of the algorithm), $W_t$ equals
\begin{equation*}
\sum_{k=1}^K w_{k,t}
= \sum_{k=1}^K w_{k,t-1} \Bigl( 1+\eta_k r_{k,t}  \Bigr) \\
 =  W_{t-1} + \biggl( \underbrace{\sum_{k=1}^K \eta_k w_{k,t-1}}_{= \b \eta^\top \b w_{t-1}} \biggr) \algloss_t - \sum_{k=1}^K \underbrace{\eta_k w_{k,t-1}}_{= \b \eta^\top \b w_{t-1} \, p_{k,t}} \ell_{k,t}\,.
\end{equation*}
Substituting the definition of $\bp_t$ (step 1 of the algorithm), as indicated in the line above,
the last two sums are seen to cancel out, leading to $W_t = W_{t-1}$.
Combining the lower bound on $\ln W_T$ with its value $0$ and rearranging concludes the proof.
\end{proof}

\section{Algorithms and bound for parameters varying over time}
\label{sec:TuningEta}

To achieve the optimized bound~\eqref{eq:optiSOB}, the
learning parameters $\eta_k$ must be tuned using preliminary
knowledge of the sums $\sum_{t=1}^T \bigl(\algloss_t - \ell_{k,t}
\bigr)^2$. In this section we show how to remove this requirement, at
the cost of a logarithmic factor $\ln \ln T$ only (unlike what would be
obtained by working in regimes as mentioned above). We do so by having
the learning rates $\eta_{k,t}$ for each expert vary with time.

\subsection{Multiplicative updates (adaptive version of ML-Prod)}

We generalize Algorithm~\ref{algo:prod} and
Theorem~\ref{th:lmprod-plain} to Algorithm~\ref{algo:prod-adapt} and
Theorem~\ref{th:lmprod-adapt}.

\begin{theorem}
\label{th:lmprod-adapt}
For all sequences of loss vectors $\b \ell_t \in [0,1]^K$, for all rules
prescribing sequences of learning rates $\eta_{k,t} \leq 1/2$ that, for each $k$, are nonincreasing in $t$,
Algorithm~\ref{algo:prod-adapt} ensures
\begin{align*}
\hspace{-0.25cm}
\min_{1 \leq k \leq K} \Biggl\{ &\sum_{t=1}^T \ell_{k,t} +
\frac{1}{\eta_{k,0}} \ln \frac{1}{w_{k,0}} + \sum_{t=1}^T \eta_{k,t-1} \bigl(\algloss_t - \ell_{k,t} \bigr)^2
+ \frac{1}{\eta_{k,T}} \ln \left( 1 + \frac{1}{\e}
\sum_{k'=1}^K \sum_{t=1}^T \left( \frac{\eta_{k',t-1}}{\eta_{k',t}} - 1 \right) \right)   \Biggr\}.
\end{align*}
\end{theorem}

\begin{algorithm}[thb]
\caption{Prod with multiple adaptive learning rates (Adapt-ML-Prod)}
\label{algo:prod-adapt}
	\smallskip
    	\emph{Parameter}: a rule to sequentially pick the learning rates \\
	\emph{Initialization}: a vector $\b w_0 =
        (w_{1,0},\ldots,w_{K,0})$ of nonnegative weights that sum to $1$

        \smallskip
	\emph{For} each round $t = 1,\,2,\,\ldots$ \\
    \indent 0. pick the learning rates $\eta_{k,t-1}$ according to the rule \\
	\indent 1. form the mixture $\bp_t$ defined component-wise by
	$\displaystyle{p_{k,t} = {\eta_{k,t-1} w_{k,t-1}} \big/ \b \eta_{t-1}^\top \b w_{t-1}}$ \\
	\indent 2. observe the loss vector $\b \ell_t$ and incur loss $\algloss_t = \b p_t^\top \b \ell_t$\\
	\indent 3. for each expert $k$ perform the update \\
		\indent \indent \indent $\displaystyle{w_{k,t} =  \biggl( w_{k,t-1} \Bigl( 1+\eta_{k,t-1} \bigl(\algloss_t - \ell_{k,t} \bigr) \Bigr)
\biggr)^\frac{\eta_{k,t}}{\eta_{k,t-1}}}$
\end{algorithm}

\begin{corollary}
\label{cor:lmprod-adapt}
With uniform initial
weights $\b w_0 = (1/K,\ldots,1/K)$ and
learning rates, for $t \geq 1$,
\[
\eta_{k,t-1} = \min \left\{ \frac{1}{2}, \,\, \sqrt{\frac{\ln K}{1+\sum_{s=1}^{t-1}
\bigl(\algloss_s - \ell_{k,s} \bigr)^2}} \right\},
\]
the cumulative loss of Algorithm~\ref{algo:prod-adapt} is bounded by
\[
 \min_{1 \leq k \leq K} \Biggl\{ \sum_{t=1}^T \ell_{k,t} +
\frac{C_{K,T}}{\sqrt{\ln K}} \sqrt{1+\sum_{t=1}^{T} \bigl(\algloss_t - \ell_{k,t} \bigr)^2}  + 2 C_{K,T} \Biggr\},
 \]
 where $C_{K,T} = 3\ln K + \ln \displaystyle{\!\left( 1 + \frac{K}{2\e} \bigl(1 + \ln (T+1) \bigr)
\right)} = O(\ln K + \ln\ln T)$.
\end{corollary}

This optimized corollary is the adaptive version
of~\eqref{eq:optiSOB}. Its proof is postponed to
Section~\ref{app:Proof-lmprod-adapt} of the additional material.
Here we only give the main ideas in the proof of
Theorem~\ref{th:lmprod-adapt}. The complete argument is given in
Section~\ref{app:ProofLmprod-Adapt} of the additional material. We point
out that the proof technique is not a routine adaptation of well-known
tuning tricks such as, for example, the ones of \citet{AuerEtAl2002}.

\begin{proof}[\textbf{sketch for Theorem~\ref{th:lmprod-adapt}}]
We follow the path of the proof of Theorem~\ref{th:lmprod-plain} and bound $\ln W_T$
from below and from above.
The lower bound is easy to establish as it only relies on individual non-increasing sequences of rates, $(\eta_{k,t})_{t \geq 0}$
for a fixed $k$: the weight update (step 3 of the algorithm) was indeed tailored for it to go through.
More precisely, by induction and still with the inequality $\ln(1+x) \geq x-x^2$ for $x \geq -1/2$, we get that
\[
\ln W_T \geq \ln w_{k,T} \geq \frac{\eta_{k,T}}{\eta_{k,0}} \ln w_{k,0} +
\eta_{k,T} \sum_{t=1}^T \bigl( r_{k,t} - \eta_{k,t-1} r_{k,t}^2 \bigr)\,.
\]
The difficulties arise in proving an upper bound.
We proceed by induction again and aim at upper bounding $W_{t}$ by $W_{t-1}$ plus some small term.
The core difficulty is that the powers $\eta_{k,t}/\eta_{k,t-1}$ in the weight update are different for each $k$.
In the literature, time-varying parameters could previously be handled
using Jensen's inequality
for the function $x \mapsto x^{\alpha_t}$ with a parameter $\alpha_t = \eta_{t}/\eta_{t-1} \geq 1$
that was the same for all experts: this is, for instance,
the core of the argument in the main proof of~\citet{AuerEtAl2002} as noticed by~\citet{OttGy} in their re-worked version
of the proof. This needs to be adapted here as we have $\alpha_{k,t} =
\eta_{k,t-1}/\eta_{k,t}$, which depends on $k$. We quantify the cost for the $\alpha_{k,t}$ not to be all equal to a single power $\alpha_t$,
say $1$: we have $\alpha_{k,t} \geq 1$ but the gap to $1$ should not be too large. This is why we may apply
the inequality $x \leq x^{\alpha_{k,t}} + (\alpha_{k,t}-1)/e$, valid
for all $x>0$ and $\alpha_{k,t} \geq 1$. We can then prove that
\vspace{-.9cm}

\[
\hspace{7cm} W_{t} \leq W_{t-1} + \frac{1}{e} \sum_{k=1}^K \left(\frac{\eta_{k,t-1}}{\eta_{k,t}}-1\right)\,,
\]
where the second term on the right-hand side is precisely the price to
pay for having different time-varying learning rates --- and
this price is measured by how much they vary. \vspace{-.6cm}
\end{proof}

\subsection{Polynomial potentials}

As illustrated in~\citet{CeLu03}, polynomial potentials
are also useful to minimize the regret. We present here an algorithm
based on them (with order $p = 2$ in the terminology of the indicated reference).
Its bound has the same poor dependency on the number of experts $K$ and
on $T$ as achieved by working in regimes (see the discussion in
Section~\ref{sec:MLProd}), but its analysis is simpler and more elegant
than that of Algorithm~\ref{algo:prod-adapt} (see
Section~\ref{app:Proof-polpot-adapt} in the appendix; the analysis
resembles the proof of Blackwell's approachability theorem). The right
dependencies might be achieved by considering polynomial functions of
arbitrary orders $p$ as in~\citet{CeLu03}.

\begin{algorithm}[thb]
\caption{Polynomially weighted averages with multiple learning rates (ML-Poly)}
\label{algo:polpot-adapt}
	\smallskip
    	\emph{Parameter}: a rule to sequentially pick the learning rates $\b \eta_{t} = \bigl(\eta_{1,t},\ldots,\eta_{K,t} \bigr)$ \\
	\emph{Initialization}: the vector of regrets with each expert $\b R_0 = (0,\dots,0)$

        \smallskip
	\emph{For} each round $t = 1,\,2,\,\ldots$ \\
    	\indent 0. pick the learning rates $\eta_{k,t-1}$ according to the rule \\
	\indent 1. form the mixture $\bp_t$ defined component-wise by
	$p_{k,t} = \eta_{k,t-1} \left(R_{k,t-1}\right)_+  \ /\ \b \eta_{t-1}^\top \left(\b R_{t-1}\right)_+$ \\
	\indent \indent where $\b x_+$ denotes the vector of the nonnegative parts of the components of $\b x$ \\
	\indent 2. observe the loss vector $\b \ell_t$ and incur loss $\algloss_t = \b p_t^\top \b \ell_t$\\
	\indent 3. for each expert $k$ update the regret:
	$R_{k,t} = R_{k,t-1} + \algloss_t - \ell_{k,t}$
\end{algorithm}

\begin{theorem} \label{th:polpot-adapt}
For all sequences of loss vectors $\b \ell_t \in [0,1]^K$, the cumulative loss of
Algorithm~\ref{algo:polpot-adapt} run with learning rates
\[
\eta_{k,t-1} = {\frac{1}{1+\sum_{s=1}^{t-1} \bigl( \algloss_s- \ell_{k,s} \bigr)^2}}
\]
is bounded by $\displaystyle{\hspace{1cm}
\sum_{t=1}^T \algloss_t \leq
\min_{1 \leq k \leq K} \Biggl\{ \sum_{t=1}^T \ell_{k,t} +
\sqrt{ K \bigl(  1+\ln \!\left(1+ T\right)\bigr) \left(1+\sum_{t=1}^{T} \bigl(\algloss_t - \ell_{k,t} \bigr)^2\right)}\Biggr\} \,.
}$
\end{theorem}

\section{First application: bounds with experts that report their confidences}
\label{sec:appliconf}
\label{sec:descr}
\label{sec:genred}

We justify in this section why the second-order bounds exhibited in the previous sections
are particularly adapted to the setting of prediction with experts that
report their confidences, which was first
considered\footnote{\label{fn:1} Technically, \citet{BlumMansour2007}
decouple the confidences $I_{k,t}$, which they refer to as ``time
selection functions'', from the experts, but as explained in
Section~\ref{sec:BM} the two settings are equivalent.}. It differs from
the standard setting in that, at the start of every round $t$, each
expert $k$ expresses its confidence
as a number $I_{k,t} \in [0,1]$.
In particular, confidence $I_{k,t} = 0$ expresses that expert $k$ is
inactive (or sleeping) in round $t$.
The learner now has to assign nonnegative weights ${\bp}_t$, which sum up
to $1$, to the set $\cA_t = \{k : I_{k,t} > 0\}$
of so-called active experts
and suffers loss $\algloss_t = \sum_{k \in \cA_t} p_{k,t} \loss_{k,t}$.
(It is assumed that, for any round $t$, there is at
least one active expert $k$ with $I_{k,t} > 0$, so that $\cA_t$ is never
empty.)

The main difference in prediction with confidences comes from the definition of the regret.
The \emph{confidence regret} with respect to expert $k$ takes the numbers $I_{k,t}$ into account and
is defined as $\cregret_{k,T} = \sum_{t=1}^T I_{k,t} \bigl( \algloss_t - \loss_{k,t} \bigr)$.

When $I_{k,t}$ is always $1$, prediction with confidences reduces to
regular prediction with expert advice, and when the confidences
$I_{k,t}$ only take on the values $0$ and $1$, it reduces to
prediction with sleeping (or specialized) experts as introduced by
\citet{Blum1997} and \citet{FreundSchapireSingerWarmuth1997}.

Because the confidence regret scales linearly with $I_{k,t}$,
one would therefore like to obtain bounds on the confidence regret that
scale linearly as well. When confidences do not depend on $k$, this is
achieved, e.g., by the bound~\eqref{eq:ExponentialWeightsVariances}.
However, for confidences that do depend on $k$, the best available stated bound
\citep[Theorem~16]{BlumMansour2007} is
\begin{equation}
	\label{eq:BlumMansour2007}
  \cregret_{k,T} = \sum_{t=1}^T I_{k,t} \bigl( \algloss_t - \loss_{k,t} \bigr) =
  O\left(\sqrt{\sum_{t \leq T} I_{k,t} \loss_{k,t}}\right).
\end{equation}
(We rederive this bound in Section~\ref{sec:BM} of the supplementary
material.) If, in this bound, all confidences $I_{k,t}$ are scaled down
by a factor $\lambda_k \in [0,1]$, then we would like the bound to also
scale down by $\lambda_k$, but instead it scales only by
$\sqrt{\lambda_k}$. In the remainder of this section we will show how
our new second-order bound \eqref{eq:NewBound} solves this issue via a
generic reduction of the setting of prediction with confidences to the
standard setting from Sections~\ref{sec:introduction}
and~\ref{sec:MLProd}.

\begin{remark}
We consider the case of linear losses. The extension of our results to convex losses
is immediate via the so-called gradient trick. The latter also applies in the setting of
experts that report their confidences. The details were (essentially) provided by \citet{DevaineGaillardGoudeStoltz2013}
and we recall them in Section~\ref{sec:gradient} of
the supplementary material.
\end{remark}

\paragraph{Generic reduction to the standard setting}

There exists a generic reduction from the setting of sleeping experts to
the standard setting of prediction with expert advice
\citep{AdamskiyKoolenChernovVovk2012,KoolenAdamskiyWarmuth2013}.
This reduction generalizes easily to the setting of experts that report
their confidences, as we will now explain.

Given any algorithm designed for the standard setting, we run it on
modified losses $\lt_{k,s}$, which will be defined shortly. At round $t
\geq 1$, the algorithm takes as inputs the past modified losses
$\lt_{k,s}$, where $s \leq t-1$, and outputs a weight vector
$\bpt_t$ on $\{1,\ldots,K\}$. This vector is then used to form another
weight vector $\bp_t$, which has strictly positive weights only on
$\cA_t$:
\begin{equation}
\label{eq:redistrsingle}
p_{k,t} = \frac{I_{k,t} \, \pt_{k,t}}{\sum_{k'=1}^K I_{k',t} \, \pt_{k',t}}
\qquad \mbox{for all} \ k.
\end{equation}
This vector $\bp_t$ is to be used with the experts that report their
confidences.
Then, the losses $\ell_{k,t}$ are observed and the modified losses are computed as follows: for all $k$,
\[
\lt_{k,t} = I_{k,t} \ell_{k,t} + (1-I_{k,t}) \algloss_t
\qquad \mbox{where} \quad \algloss_t = \sum_{k \in \cA_t} p_{k,t} \loss_{k,t}\,.
\]

\begin{proposition}\label{prop:reduction}
The induced confidence regret on the original losses $\ell_{k,t}$ equals
the standard regret of the algorithm on the modified losses $\lt_{k,t}$.
In particular,
\[
I_{k,t} \bigl( \algloss_t - \loss_{k,t} \bigr) = \sum_{i=1}^K \pt_{i,t}\lt_{i,t} - \lt_{k,t}
\qquad \text{for all rounds $t$ and experts $k$.}
\]
\end{proposition}

\begin{proof}
First we show that the loss in the standard setting (on the losses
$\lt_{k,t}$) is equal to the loss in the confidence regret setting (on
the original losses $\ell_{k,t}$):
\begin{align*}
\sum_{k=1}^K \pt_{k,t}\lt_{k,t}
  &= \sum_{k=1}^K \pt_{k,t} \Bigl( I_{k,t} \ell_{k,t} + (1-I_{k,t})
  \algloss_t \Bigr)
  = \sum_{k=1}^K \pt_{k,t} I_{k,t} \ell_{k,t}
    + \algloss_t
    - \left(\sum_{k=1}^K \pt_{k,t}I_{k,t} \right) \algloss_t\\
  &= \left( \sum_{k'=1}^K \pt_{k',t}\,I_{k',t} \!\right)\!
  \sum_{k=1}^K p_{k,t} \ell_{k,t} 
    + \algloss_t
    - \left(\sum_{k=1}^K \pt_{k,t}I_{k,t} \right) \algloss_t
  = \algloss_t.
\end{align*}
The proposition now follows by subtracting $\lt_{k,t}$ on both sides of
the equality.
\end{proof}
\vspace{-.8cm}

\begin{corollary}
\label{cor:Isq}
An algorithm with a standard regret bound of the form
\begin{equation}
\label{eq:init}
R_{k,T}
\leq \Xi_1 \sqrt{(\ln K) \sum_{t \leq T} \bigl(\algloss_t - \ell_{k,t} \bigr)^2} + \Xi_2
\qquad \mbox{for all} \ k,
\end{equation}
leads, via the generic reduction described above (and for losses $\ell_{k,t} \in [0,1]$),
to an algorithm with a confidence regret bound of the form
\begin{equation}
\label{eqn:SquaredConfidenceBound}
\cregret_{k,T}
\leq \Xi_1 \sqrt{(\ln K) \sum_{t \leq T} I_{k,t}^2 \bigl(\algloss_t - \ell_{k,t} \bigr)^2} + \Xi_2
\leq \Xi_1 \sqrt{(\ln K) \sum_{t \leq T} I_{k,t}^2} + \Xi_2
\qquad \text{for all $k$.}
\end{equation}
\end{corollary}

We note that the second upper-bound, $\sqrt{\sum I_{k,t}^2}$, can be extracted from the
proof of Theorem~11 in \citet{ChernovVovk2010}---but not the first one, which,
combined with the techniques of Section~\ref{sec:EntailedImprovementSmallLosses},
yields a bound on the confidence regret for small (excess) losses.

\paragraph{Comparison to the instantiation of other regret bounds}
We now discuss why \eqref{eqn:SquaredConfidenceBound} improves on the literature.
Consider first the improved bound for small losses from the introduction,
which takes the form $\Xi_3 \sqrt{\sum_{t} \ell_{k,t}} + \Xi_4$.
This improvement does not
survive the generic reduction, as the resulting confidence regret bound is
\[
\Xi_3 \sqrt{\sum_{t=1}^T \lt_{k,t}} + \Xi_4
= \Xi_3 \sqrt{\sum_{t=1}^T I_{k,t} \ell_{k,t} + \underbrace{\sum_{t=1}^T
(1-I_{k,t}) \algloss_t}_{\mbox{\tiny undesirable}}} + \Xi_4,
\]
which is no better than plain $\Xi'_3 \sqrt{T} + \Xi'_4$ bounds.

Alternatively, bounds~\eqref{eq:ExponentialWeightsVariances} of \citet{CesaBianchiMansourStoltz2007} and \citet{DeRooijEtAl2013}
are of the form
\begin{equation*}
\Xi_5 \sqrt{\sum_{t=1}^T \sum_{k=1}^K p_{k,t} \bigl( \ell_{k,t} - \hat
\ell_t\bigr)^2} + \Xi_6,
\end{equation*}
uniformly over all experts $k$.
These lead to a confidence regret bound against expert $k$ of the form
\begin{equation*}
\Xi_5  \sqrt{\sum_{t=1}^T \sum_{k=1}^K p_{k,t} \, I_{k,t}^2  \bigl( \algloss_t - \loss_{k,t} \bigr)^2} + \Xi_6
  \leq \Xi_5 \sqrt{\sum_{t=1}^T \sum_{k=1}^K p_{k,t} \, I_{k,t}^2} + \Xi_6,
\end{equation*}
which depends not just on the confidences of this expert $k$, but also
on the confidences of the other experts. It therefore does not scale
proportionally to the confidences of the expert $k$ at hand.

We note that even bounds of the form~\eqref{eq:impossibletuning}, if they existed, would not
be suitable either. They would indeed lead to
\[
	R_{k,T}^c = O \left( \sqrt{\sum_{t=1}^T \left( I_{k,t} \ell_{k,t} + (1-I_{k,t}) \algloss_t\right)^2} \right),
\]
which also does not scale linearly with the confidences of expert $k$.

\section{Other applications: bounds in the standard setting}
\label{sec:OtherApplications}

We now leave the setting of prediction with confidences, and detail other applications of our
new second-order bound \eqref{eq:NewBound}. First,
in Section~\ref{sec:EntailedImprovementSmallLosses}, we show that, like
\eqref{eq:TunedProd2007} and~\eqref{eq:ExponentialWeightsVariances}, our
new bound implies an improvement over the standard bound
$O\big(\sqrt{\sum \ell_{k,t} \ln K}\big)$, which is itself already better than
the worst-case bound if the losses of the reference expert are small.
The key feature in our improvement is that excess losses $\ell_{k,t}
- \hat{\ell}_t$ can be considered instead of plain losses $\ell_{k,t}$.
Then, in Section~\ref{sec:iid}, we look at the non-adversarial setting in
which losses are i.i.d., and show that our new bound implies constant regret of
order $O\big(\ln K\big)$.

\subsection{Improvement for small excess losses}
\label{sec:EntailedImprovementSmallLosses}

It is known \citep{CesaBianchiMansourStoltz2007,DeRooijEtAl2013}
that \eqref{eq:ExponentialWeightsVariances} implies a bound of the form
\begin{equation}\label{eq:EWSmallLosses}
  R_{k^*,T} = O\left(\sqrt{\ln K
  \frac{L_{k^*,T}(T-L_{k^*,T})}{T}}\right),
\end{equation}
where $k^* \in \argmin_k L_{k,T}$ is the expert with smallest cumulative
loss. This bound symmetrizes the standard bound for small losses
described in the introduction, because it is small also if $L_{k^*,T}$
is close to $T$, which is useful when losses are defined in terms of
gains \citep{CesaBianchiMansourStoltz2007}.

However, if one is ready to lose symmetry, another way of improving the standard bound
for small losses is to express it in terms of \emph{excess} losses:
\[
\sqrt{ \ln K \sum_{t \colon \ell_{k,t}\geq\algloss_t} \left(\ell_{k,t}-\algloss_t\!\right)}
\leq \sqrt{ \ln K \sum_{t \leq T} \ell_{k,t} }\,,
\]
where the inequality holds for nonnegative losses.
As we show next, bounds of the form~\eqref{eq:NewBound} indeed entail bounds of this form.

\begin{theorem}
If the regret of an algorithm satisfies~\eqref{eq:init}
for all sequences of loss vectors $\b \ell_t \in [0,1]^K$, then it also satisfies
\begin{equation}
\label{eq:ISEL}
R_{k,T} \leq 2\, \Xi_1 \sqrt{ \ln K \sum_{t \colon \ell_{k,t}\geq\algloss_t} \left(\ell_{k,t}-\algloss_t\!\right)}
+ \bigl( \Xi_2 + 2\,\Xi_1 \sqrt{\Xi_2 \ln K} + 4\,\Xi_1^2 \ln K \bigr)\,.
\end{equation}
\end{theorem}

In general, losses take values in the range $[a,b]$. To apply our
methods, they therefore need to be translated by $-a$ and scaled by
$1/(b-a)$ to fit the canonical range $[0,1]$. In the standard
improvement for small losses, these operations remain visible in the
regret bound, which becomes $R_{k,T} = O\big(\sqrt{(b-a) (L_{k,T}-Ta)
\ln K}\big)$ in general. In particular, if $a < 0$, then no significant
improvement over the worst-case bound $O\big(\sqrt{T \ln K}\big)$ is
realized. By contrast, our original second-order bound \eqref{eq:init}
and its corollary \eqref{eq:ISEL} both have the nice feature that
translations do not affect the bound because
$(\ell_{k,t}-a)-(\algloss_t-a) = \ell_{k,t}-\algloss_t$, so that our new
improvement for small losses remains meaningful even for $a < 0$.

\begin{proof}
We define the positive and the negative part of the regret with respect to an expert $k$ by, respectively,
\vspace{-.6cm}

\[
\hspace{2cm}
	R_{k,T}^+  =  \sum_{t=1}^{T} \left(\algloss_t - \ell_{k,t}\!\right) \indic{\ell_{k,t}\leq\algloss_t}
\qquad \mbox{ and}  \qquad
	R_{k,T}^-  =  \sum_{t=1}^{T} \left(\!\ell_{k,t}-\algloss_t\!\right) \indic{\ell_{k,t}\geq\algloss_t} \,.
\]
The proof will rely on rephrasing the bound~\eqref{eq:init} in terms of $R_{k,T}^+$ and $R_{k,T}^-$ only.
On the one hand, $R_{k,T} = R_{k,T}^+ - R_{k,T}^-$, while, on the other hand,
\begin{equation}
\label{eq:smalllosses2}
\sqrt{\sum_{t \leq T} \bigl(\algloss_t - \ell_{k,t} \bigr)^2}
\leq \sqrt{ \sum_{t \leq T} \left|  \algloss_t - \ell_{k,t}\right|}
=  \sqrt{ R_{k,T}^+ + R_{k,T}^- }
\leq 2\sqrt{ R_{k,T}^+ }\,,
\end{equation}
where we used $\ell_{k,t} \in [0,1]$ for the first inequality
and where we assumed, with no loss of generality, that $R_{k,T}^+ \geq R_{k,T}^-$.
Indeed, if this was not the case, the regret would be negative and the bound would be true.
Therefore for all experts $k$, substituting these (in)equalities in the initial inequality~\eqref{eq:init},
we are left with the quadratic inequality
\begin{equation}
\label{eq:smalllosses3}
R_{k,T}^+ - R_{k,T}^- \leq 2 \Xi_1 \sqrt{R_{k,T}^+\ln K } + \Xi_2\,.
\end{equation}
Solving for $R_{k,T}^+$ using Lemma~\ref{lem:QuadracticInequality} below (whose proof can
be found in Section~\ref{app:ProofQuadraticInequality}) yields
\[
\sqrt{R_{k,T}^+} \leq \sqrt{R_{k,T}^- + \Xi_2} + 2 \Xi_1 \sqrt{\ln K}
\leq \sqrt{R_{k,T}^-} + \sqrt{\Xi_2} + 2 \Xi_1 \sqrt{\ln K}\,,
\]
which leads to the stated bound after re-substitution into~\eqref{eq:smalllosses3}.
\end{proof}
\vspace{-.75cm}

\begin{lemma}\label{lem:QuadracticInequality}
Let $a,c \geq 0$. If $x \geq 0$ satisfies $x^2 \leq a + cx$, then $x \leq \sqrt{a}+c$.
\end{lemma}

\subsection{Stochastic (i.i.d.) losses}\label{sec:iid}

\begin{sloppypar}
\Citet{ErvenGKR11} provide a specific algorithm that guarantees worst-case
regret bounded by $O\left(\sqrt{L_{k^\star,T} \ln K}\right)$, but at the
same time is able to adapt to the non-adversarial setting with
independent, identically distributed (i.i.d.)\ loss vectors, for which
its regret is bounded by $O(K)$. In the previous section we have already
discussed how any algorithm satisfying a regret bound of the form
\eqref{eq:init} also achieves a worst-case bound that is at least as
good as $O\left(\sqrt{L_{k^\star,T} \ln K}\right)$. Here we consider
i.i.d.\ losses that satisfy the same assumption as the one imposed by
\Citeauthor{ErvenGKR11}:
\end{sloppypar}
\begin{ass}
\label{ass:iid}
The loss vectors $\b \ell_t \in [0,1]^K$ are independent random
variables such that there exists an action $k^\star$ and some $\alpha
\in (0,1]$ for which the expected differences in loss satisfy
\[
\forall t \geq 1, \qquad \min_{k \ne k^\star} \, \E\bigl[ \ell_{k,t} - \ell_{k^\star,t} \bigr] \geq \alpha\,.
\]
\end{ass}
As shown by the following theorem, any algorithm that satisfies our new
second-order bound (with a constant $\Xi_1$ factor and a $\Xi_2$ factor of order $\ln K$)
is guaranteed to achieve constant regret of order
$O(\ln K)$ under Assumption~\ref{ass:iid}.
%
%
\begin{theorem}\label{th:iid}
If a strategy achieves a regret bound of the form~\eqref{eq:init} and
the loss vectors satisfy Assumption~\ref{ass:iid}, then
the expected regret for that strategy is bounded by a constant,
\[
\E[R_{k^\star,T}] \leq C(\Xi_1,\Xi_2,\alpha) \stackrel{{\rm \tiny def}}{=}
(\Xi_1^2 \ln K)/\alpha + \Xi_1\sqrt{(\Xi_2 \ln K)/\alpha} + \Xi_2\,,
\]
while for all $\delta \in (0,1)$, its regret is bounded with probability at least $1-\delta$ by
\[
R_{k^\star,T} \leq C(\Xi_1,\Xi_2,\alpha)  + \frac{6 \, \Xi_1}{\alpha} \sqrt{\Biggl( \ln \frac{1}{\delta}
+ \ln \!\left( 1 + \frac{1}{2\e} \ln \bigl(1+C(\Xi_1,\Xi_2,\alpha)/4 \bigr) \right)\Biggr) \ln K}\,.
\]
\end{theorem}

By the law of large numbers, the cumulative loss of any action $k \neq
k^\star$ will exceed the cumulative loss of $k^\star$ by a linear term
in the order of $\alpha T$, so that, for all sufficiently large $T$, the
fact that $R_{k^\star,T}$ is bounded by a constant implies that the
algorithm will have negative regret with respect to all other $k$.
\medskip

Because we want to avoid using any special properties of the algorithm
except for the fact that it satisfies \eqref{eq:init}, our proof of
Theorem~\ref{th:iid} requires a Bernstein-Freedman-type martingale
concentration result \citep{Freed} rather than basic applications of
Hoeffding's inequality, which are sufficient in the proof of
\Citet{ErvenGKR11}.
However, this type of concentration inequalities is typically stated in
terms of an a priori deterministic bound $M$ on the cumulative conditional variance $\sum V_t$.
To bound the deviations by the (random) quantity $\sqrt{\sum V_t}$
instead of the deterministic $\sqrt{M}$,
peeling techniques can be applied as in \citet[Corollary~16]{CesaBianchiLugosiStoltz2005};
this leads to an additional $\sqrt{\ln T}$ factor (in case of an additive peeling) or $\sqrt{\ln \ln T}$
(in case of a geometric peeling).
Here, we replace these non-constant factors by a term of order $\ln
\ln \E\bigl[\sum V_t\bigr]$, which will be seen to be less than a
constant in our case.

\begin{theorem}
\label{th:Bern}
Let $(X_t)_{t \geq 1}$ be a martingale difference sequence with respect to some filtration $\cF_0 \subseteq \cF_1 \subseteq
\cF_2 \subseteq \ldots$ and let $V_t = \E\big[ X_t^2 \,\big| \, \cF_{t-1} \big]$ for $t \geq 1$.
We assume that $X_t \leq 1$ a.s., for all $t \geq 1$. Then, for all $\delta \in (0,1)$ and for all $T \geq 1$,
with probability at least $1-\delta$,
\[
\sum_{t=1}^T X_t \leq 3\sqrt{\left(1+\sum_{t=1}^T V_t\right) \ln \frac{\gamma}{\delta}} + \ln \frac{\gamma}{\delta}\,,
\quad \mbox{where} \quad
\gamma = \displaystyle{1 + \frac{1}{2\e} \! \left(1 + \ln\left( 1 + \E\!\left[\sum_{t=1}^T V_t\right] \right)\right)}.
\]
\end{theorem}

Theorem~\ref{th:Bern} and its proof (see Section~\ref{sec:Bern}) may be
of independent interest, because our derivation uses new techniques that we
originally developed for time-varying learning rates in the proof of
Theorem~\ref{th:lmprod-adapt}. Instead of studying supermartingales of
the form $\exp\bigl( \lambda \sum X_t - (\e-2)\lambda^2 \sum V_t \bigr)$
for some constant value of $\lambda$, as is typical, we are able to
consider (predictable) random variables $\Lambda_t$,
which in some sense play the role of the time-varying learning parameter
$\eta_t$ of the (ML-)Prod algorithm.

\begin{proof}\textbf{[of Theorem~\ref{th:iid}]}
We recall the notation $r_{k,t} = \algloss_t - \ell_{k,t}$ for the instantaneous regret.
We define $\cF_0$ as the trivial $\sigma$--algebra $\{\emptyset,\Omega\}$ and define by induction
the following martingale difference sequence: for all $t \geq 1$,
\[
Y_t = - r_{k^\star,t} + \E\big[ r_{k^\star,t} \,\big| \, \cF_{t-1} \big] \qquad \mbox{and} \qquad
\cF_t = \sigma(Y_1,\ldots,Y_t)\,.
\]
We start by bounding the expectation of the regret. We first note that
\begin{equation}
\label{eq:lba}
\E\big[ r_{k^\star,t} \,\big| \, \cF_{t-1} \big] = \sum_{k=1}^K p_{k,t} \, \E\bigl[ \ell_{k,t} - \ell_{k^\star,t} \,\big| \, \cF_{t-1} \bigr]
= \sum_{k=1}^K p_{k,t} \, \E\bigl[ \ell_{k,t} - \ell_{k^\star,t} \bigr]
\geq \alpha (1-p_{k^\star,t})\,,
\end{equation}
\vspace{-1.2cm}

\begin{align}
\label{eq:Wt}
\mbox{while by convexity of} \ (\,\cdot\,)^2, \hspace{1.5cm} & r_{k^\star,t}^2 \leq
\sum_{k=1}^K p_{k,t} \, \bigl( \ell_{k,t} - \ell_{k^\star,t} \bigr)^2
\leq 1-p_{k^\star,t}\,, \\
\label{eq:Wt2}
\mbox{thus} \hspace{4.5cm} & W_t = \E\big[ Y_t^2 \,\big| \, \cF_{t-1} \big] \leq
\E\big[ r_{k^\star,t}^2 \,\big| \, \cF_{t-1} \big]
\leq 1-p_{k^\star,t}\,.
\end{align}
Therefore, using that expectations of conditional expectations are unconditional expectations,
\begin{equation}
\label{eq:Eregr}
\E[R_{k^\star,T}] \geq \alpha \, \E[S_T]
\quad \mbox{and} \quad
\E\!\left[\sum_{t=1}^{T} r_{k^\star,t}^2 \right] \leq \E[S_T]
\qquad \mbox{where} \quad
\displaystyle{S_T = \sum_{t=1}^T (1-p_{k^\star,t})}\,.
\end{equation}
Substituting these inequalities in~\eqref{eq:init} using Jensen's inequality for $\sqrt{\,\cdot\,}$,
we get
\[
\E[S_T] \leq \frac{\Xi_1 \sqrt{\ln K}}{\alpha} \sqrt{\E[S_T]} + \frac{\Xi_2}{\alpha}\,.
\]
Solving the quadratic inequality (see Lemma~\ref{lem:QuadracticInequality})
yields $\E[S_T] \leq \bigl( (\Xi_1 \sqrt{\ln K})/\alpha + \sqrt{\Xi_2/\alpha} \bigr)^2$.
By \eqref{eq:Eregr} this bounds $\E\!\left[\sum_{t=1}^{T}
r_{k^\star,t}^2 \right]$, which we substitute into~\eqref{eq:init}, together with Jensen's inequality, to prove
the claimed bound on the expected regret.
\medskip

Now, to get the high-probability bound, we apply Theorem~\ref{th:Bern}
to $X_t = Y_t/2 \leq 1$ a.s.\ and $V_t = W_t/4$ and use the
bounds~\eqref{eq:lba} and~\eqref{eq:Wt2}. We find that, with probability at least
$1-\delta$,
\[
\alpha S_T \leq R_{k^\star,T} + 3\sqrt{(4+S_T) \ln(\gamma/\delta)} + 2\ln(\gamma/\delta)
\leq R_{k^\star,T} + 3\sqrt{S_T \ln(\gamma/\delta)} + 8\ln(\gamma/\delta)
\]
where $\displaystyle{\gamma \leq 1 + (1/2\e) \,\Bigl[1 + \ln \bigl( 1+\E[S_T]/4 \bigr) \Bigr]}$
and where we used $\sqrt{\ln(\gamma/\delta)} \geq 1$.
Combining the bound~\eqref{eq:init} on the regret
with~\eqref{eq:Wt}
yields $R_{k^\star,T} \leq \Xi_1 \sqrt{S_T \ln K} + \Xi_2$, so that, still with probability at least $1-\delta$,
\[
\alpha S_T
\leq \left(\Xi_1 \sqrt{\ln K} + 3\sqrt{\ln(\gamma/\delta)} \right) \sqrt{S_T} + \Bigl( 8\ln(\gamma/\delta) + \Xi_2 \Bigr)\,.
\]
Solving for $\sqrt{S_T}$ with Lemma~\ref{lem:QuadracticInequality} and
using that $\alpha \leq 1$, this implies
\[
\sqrt{S_T} \leq \frac{\Xi_1 \sqrt{\ln K} + 3\sqrt{\ln(\gamma/\delta)}}{\alpha}
+ \frac{1}{\sqrt{\alpha}} \sqrt{8\ln(\gamma/\delta) + \Xi_2}
\leq \frac{\Xi_1\sqrt{\ln K}}{\alpha} + \sqrt{\frac{\Xi_2}{\alpha}} +
\frac{6}{\alpha} \sqrt{\ln \frac{\gamma}{\delta}}\,.
\]
Substitution into the (deterministic) regret bound $R_{k^\star,T} \leq
\Xi_1 \sqrt{S_T \ln K} + \Xi_2$ concludes the proof.
\end{proof}

\bibliography{sleeping-experts}

\newpage
\appendix
\begin{center}
{\Large
Additional Material for\\
\vspace{0.2\baselineskip}
``\ourtitle''}
\end{center}

We gather in this appendix several facts and results whose proofs were omitted from
the main body of the paper.

\section{Omitted proofs}

\subsection{Proof of Lemma~\ref{lem:QuadracticInequality}}
\label{app:ProofQuadraticInequality}

  Solving $x^2 \leq a + cx$ for $x$, we find that
  \begin{equation*}
    \half c - \half \sqrt{c^2 + 4a} \leq x \leq \half c + \half
    \sqrt{c^2 + 4a}\,.
  \end{equation*}
In particular, focusing on the upper bound, we get
$2x \leq c + \sqrt{c^2 + 4a} \leq c + \sqrt{c^2} + \sqrt{4a} = 2c + 2\sqrt{a}$,
which was to be shown.

\subsection{Proof of Theorem~\ref{th:lmprod-adapt}}
\label{app:ProofLmprod-Adapt}

The proof will rely on the following simple lemma.

\begin{lemma}
\label{lm:ineq}
For all $x > 0$ and all $\alpha \geq 1$, we have $x \leq x^\alpha + (\alpha-1)/\e$.
\end{lemma}

\begin{proof}
The inequality is straightforward when $x \geq 1$, so we restrict our attention to the case where $x < 1$.
The function $\alpha \mapsto x^\alpha = \e^{\alpha \ln x}$ is convex and thus is above any tangent line.
In particular, considering the value $x \ln x$ of the derivative function $\alpha \mapsto (\ln x)\,\e^{\alpha \ln x}$
at $\alpha = 1$, we get
\[
\forall \, \alpha > 0, \qquad x^\alpha - x \geq (x \ln x) \, (\alpha - 1)\,.
\]
Now, since we only consider $\alpha \geq 1$,
it suffices to lower bound $x \ln x$ for the values of interest for $x$, namely, the ones in $(0,1)$
as indicated at the beginning of the proof. On this interval, the stated
quantity is at least $-1/\e$,
which concludes the proof.
\end{proof}

We now prove Theorem~\ref{th:lmprod-adapt}. \\

\begin{proof}\textbf{[of Theorem~\ref{th:lmprod-adapt}]}
As in the proof of Theorem~\ref{th:lmprod-plain}, we bound $\ln W_T$ from below and from above. For the lower bound, we start
with $\ln W_T \geq \ln w_{k,T}$. We then show by induction that for all $t \geq 0$,
\begin{equation*}
\ln w_{k,t} \geq   \eta_{k,t} \sum_{s=1}^t \Bigl( r_{k,s}-
\eta_{k,s-1} r_{k,s}^2 \Bigr)
	+ \frac{\eta_{k,t}}{\eta_{k,0}} \ln w_{k,0}\,,
\end{equation*}
where $r_{k,s} = \algloss_s - \ell_{k,s}$ denotes the instantaneous regret with respect to expert $k$.
The inequality is trivial for $t = 0$. If it holds at a given round $t$, then
by the weight update (step 3 of the algorithm),
\begin{align*}
\ln w_{k,t+1} &=  \frac{\eta_{k,t+1}}{\eta_{k,t}} \biggl( \ln w_{k,t} +  \ln \Bigl( 1 + \eta_{k,t} r_{k,t+1}\Bigr) \biggr) \\
&	\geq   \frac{\eta_{k,t+1}}{\eta_{k,t}}
\Biggl( \frac{\eta_{k,t}}{\eta_{k,0}} \ln w_{k,0} + \eta_{k,t} \sum_{s=1}^t \Bigl( r_{k,s} -
\eta_{k,s-1} r_{k,s}^2 \Bigr) \Biggr)
	 +  \frac{\eta_{k,t+1}}{\eta_{k,t}} \bigl( \eta_{k,t}  r_{k,t+1} -
\eta_{k,t}^2  r_{k,t+1}^2 \bigr) \\
&	=  \quad \eta_{k,t+1} \sum_{s=1}^{t+1} \Bigl( r_{k,s} -
\eta_{k,s-1} r_{k,s}^2 \Bigr) + \frac{\eta_{k,t+1}}{\eta_{k,0}} \ln w_{k,0}  \,,
\end{align*}
where the inequality comes from the induction hypothesis and from the inequality
$\ln(1+x) \geq x-x^2$ for all $x \geq -1/2$ already used in the proof of Theorem~\ref{th:lmprod-plain}.

We now bound from above $\ln W_T$, or equivalently, $W_T$ itself. We show by induction that for all $t \geq 0$,
\[
W_t \leq 1 + \frac{1}{\e} \sum_{k=1}^K \sum_{s=1}^t \left( \frac{\eta_{k,s-1}}{\eta_{k,s}} - 1\right).
\]
The inequality is trivial for $t = 0$.
To show that if the property holds for some $t \geq 0$ it also holds for $t+1$, we prove that
\begin{equation}
\label{eq:inducadapt}
W_{t+1} \leq W_{t} + \frac{1}{\e} \sum_{k=1}^K \left( \frac{\eta_{k,t}}{\eta_{k,t+1}} - 1 \right).
\end{equation}
Indeed, since $x \leq x^\alpha + (\alpha-1)/\e$ for all $x > 0$ and $\alpha \geq 1$ (see Lemma~\ref{lm:ineq}),
we have, for each expert $k$,
\begin{equation}
\label{eq:weightconvineq}
w_{k,t+1} \leq  \bigl( w_{k,t+1} \bigr)^{\frac{\eta_{k,t}}{\eta_{k,t+1}}} + \frac{1}{\e} \left( \frac{\eta_{k,t}}{\eta_{k,t+1}} - 1 \right);
\end{equation}
we used here $x = w_{k,t+1}$ and $\alpha = \eta_{k,t}/\eta_{k,t+1}$, which is larger than $1$ because of the assumption that
the learning rates are nonincreasing in $t$ for each $k$. Now, by definition of the weight update (step 3 of the algorithm),
\[
\sum_{k=1}^K \bigl( w_{k,t+1} \bigr)^{\frac{\eta_{k,t}}{\eta_{k,t+1}}}
= \sum_{k=1}^K w_{k,t} \bigl( 1+\eta_{k,t} r_{k,t+1} \bigr) = W_t\,,
\]
where the second inequality follows from the same argument as in the last
display of the proof of Theorem~\ref{th:lmprod-plain}, by using that $\eta_{k,t}w_{k,t}$ is proportional to $p_{k,t+1}$.
Summing~\eqref{eq:weightconvineq} over $k$ thus yields~\eqref{eq:inducadapt} as desired.

Finally, combining the upper and lower bounds on $\ln W_T$ and rearranging leads
to the inequality of Theorem~\ref{th:lmprod-adapt}.
\end{proof}

\subsection{Proof of Corollary~\ref{cor:lmprod-adapt}}
\label{app:Proof-lmprod-adapt}

The following lemma will be useful.

\begin{lemma}
\label{lm:Riemann}
Let $a_0 > 0$  and $a_1, \, \ldots, \, a_m \in [0,1]$ be real numbers and let $f : (0,+\infty) \to [0,+\infty)$
be a nonincreasing function. Then
\[
\sum_{i=1}^m a_i \, f\bigl(a_0+\ldots+a_{i-1}\bigr)
\leq f(a_0) + \int_{a_0}^{a_0+a_1+\ldots+a_m} f(u)\d u\,.
\]
\end{lemma}

\begin{proof}
Abbreviating $s_i = a_0 + \ldots + a_i$ for $i = 0,\ldots,m$, we find that
\begin{align*}
\sum_{i=1}^m a_i \, f(s_{i-1}) & = \sum_{i=1}^m a_i \, f(s_i) + \sum_{i=1}^m a_i \big(f(s_{i-1})-f(s_i)\big) \\
& \leq \sum_{i=1}^m a_i \, f(s_i) + \sum_{i=1}^m
\big(f(s_{i-1})-f(s_i)\big)
\leq \sum_{i=1}^m a_i \, f(s_i) + f(s_0),
\end{align*}
where the first inequality follows because $f(s_{i-1}) \geq f(s_i)$ and $a_i \leq 1$ for $i \geq 1$,
while the second inequality stems from a telescoping argument together with the fact that $f(s_m) \geq 0$.
Using that $f$ is nonincreasing together
with $s_i-s_{i-1} = a_i$ for $i \geq 1$, we further have
\begin{equation*}
a_i \, f(s_i) = \int_{s_{i-1}}^{s_i} f(s_i)\d y \leq \int_{s_{i-1}}^{s_i} f(y)\d y\,.
\end{equation*}
Substituting this bound in the above inequality completes the proof.
\end{proof}

We will be slightly more general and take
\[
\eta_{k,t} = \min \left\{ \frac{1}{2}, \,\, \sqrt{\frac{\gamma_k}{1+\sum_{s=1}^{t}
r_{k,s}^2}} \right\}
\]
for some constant $\gamma_k > 0$ to be defined by the analysis.

Because of the choice of nonincreasing learning rates, the first
inequality of Theorem~\ref{th:lmprod-adapt} holds true, and the regret $R_{k,t}$ is upper-bounded by
\begin{equation}
\label{eq:thprodadapt}
\frac{1}{\eta_{k,0}} \ln \frac{1}{w_{k,0}}
+
\frac{1}{\eta_{k,T}} \ln \Biggl( 1 + \frac{1}{\e} \sum_{k'=1}^K \underbrace{\sum_{t=1}^T \left( \frac{\eta_{k',t-1}}{\eta_{k',t}}
- 1 \right)}_{\mbox{first term}} \Biggr)
+
\underbrace{\sum_{t=1}^T \eta_{k,t-1} r_{k,t}^2}_{\mbox{second term}}\,.
\end{equation}
For the first term in~\eqref{eq:thprodadapt}, we note that for each
$k'$ and $t \geq 1$ one of three possibilities must hold,
all depending on which of the inequalities in $\eta_{k',t} \leq \eta_{k',{t-1}} \leq 1/2$
are equalities or strict inequalities.
More precisely, either
$\eta_{k',t} = \eta_{k',{t-1}} = 1/2$; or
\[
\sqrt{\frac{\gamma_{k'}}{1+\sum_{s=1}^{t}
r_{k',s}^2}} =
\eta_{k',t} < \eta_{k',{t-1}} = \half \leq \sqrt{\frac{\gamma_{k'}}{1+\sum_{s=1}^{t-1}
r_{k',s}^2}}\,;
\]
or $\eta_{k',t}\leq \eta_{k',{t-1}} < 1/2$.
In all cases, the ratios $\eta_{k',t-1} / \eta_{k',t} - 1$ can be bounded as follows:
\begin{multline}
\label{eq:sumint1}
\sum_{t=1}^T \biggl( \frac{\eta_{k',t-1}}{\eta_{k',t}} - 1 \biggr)
  \leq \sum_{t=1}^T \left( \sqrt{\frac{1+\sum_{s=1}^{t}r_{k',t} ^2}{1+
\sum_{s=1}^{t-1} r_{k',t} ^2}} - 1 \right) \\
  = \sum_{t=1}^T \left( \sqrt{1 + \frac{r_{k',t} ^2}{1+
\sum_{s=1}^{t-1} r_{k',s} ^2}} - 1 \right)
  \leq  \half \sum_{t=1}^T \frac{r_{k',t} ^2}{1+
\sum_{s=1}^{t-1} r_{k',s} ^2}\,,
\end{multline}
where we used, for the second inequality, that $g(1+z) \leq g(1) +
z\,g'(1)$ for $z \geq 0$ for any concave function $g$, in particular the
square root.
We apply Lemma~\ref{lm:Riemann} with $f(x)=1/x$ to further bound the sum
in~\eqref{eq:sumint1}, which gives
\begin{equation}
\label{eq:intermpoly}
\sum_{t=1}^T \frac{r_{k',t} ^2}{
1+\sum_{s=1}^{t-1}  r_{k',s} ^2}
 \leq  1
+ \ln \!\left(1+ \sum_{t=1}^{T} r_{k',t} ^2 \right)
- \cancel{\ln \!\left(1\right)}
\leq 1+\ln (T+1)\,.
\end{equation}

For the second term in~\eqref{eq:thprodadapt}, we write
\[
\sum_{t=1}^T \eta_{k,t-1} r_{k,t} ^2
\leq \sqrt{\gamma_k} \sum_{t=1}^T \frac{r_{k,t} ^2}{
\sqrt{1+\sum_{s=1}^{t-1} r_{k,s} ^2}}\,.
\]
We apply Lemma~\ref{lm:Riemann} again, with $f(x) = 1/\sqrt{x}$, and get
\begin{align}
\label{eq:lmRsqrt}
\sum_{t=1}^T & \frac{r_{k,t} ^2}{
\sqrt{1+\sum_{s=1}^{t-1} r_{k,s} ^2}}  \leq  \   \underbrace{1 - 2\sqrt{1}}_{\leq 0} +2\sqrt{\left(1+ \sum_{t=1}^{T}r_{k,t}^2\right)}\,.
\end{align}

We may now get back to~\eqref{eq:thprodadapt}. Substituting the obtained
bounds on its first and second terms, and using $\eta_{k,0} \geq
\eta_{k,T}$, we find it is no greater than
\begin{equation}\label{eqn:intermediatecorbound}
\frac{1}{\eta_{k,T}} \left(\ln \frac{1}{w_{k,0}} + B_{K,T}\right)
+ 2\sqrt{\gamma_k \left(1+ \sum_{t=1}^{T} r_{k,t}^2 \right)}\,,
\end{equation}
where $B_{K,T} = \ln \Bigl( 1 + \frac{K}{2\e} \bigl(1 + \ln (T+1) \bigr)
\Bigr)$.

Now if $\sqrt{1+ \sum_{t=1}^{T}
r_{k,t}^2} > 2 \sqrt{\gamma_k}$ then $\eta_{k,T} < 1/2$ and~\eqref{eqn:intermediatecorbound}
is bounded by
\begin{equation*}
  \sqrt{1+ \sum_{t=1}^{T} r_{k,t}^2}\left(2\sqrt{\gamma_k} + \frac{\ln \frac{1}{w_{k,0}}
  + B_{K,T}}{\sqrt{\gamma_k}}\right).
\end{equation*}
Alternatively, if $\sqrt{1+ \sum_{t=1}^{T}
r_{k,t}^2} \leq 2 \sqrt{\gamma_k}$, then $\eta_{k,T} = 1/2$
and \eqref{eqn:intermediatecorbound} does not exceed
\begin{equation*}
2 \ln \frac{1}{w_{k,0}} + 2B_{K,T} + 4\gamma_k.
\end{equation*}
In either case, \eqref{eqn:intermediatecorbound} is smaller than the sum
of the latter two bounds, from which the corollary follows upon taking
$\gamma_k = \ln (1/w_{k,0}) = \ln K$.

\subsection{Proof of Theorem~\ref{th:polpot-adapt}}
\label{app:Proof-polpot-adapt}

The proof has a geometric flavor---the same as in the proof
of the approachability theorem \citep{Bla56}.
With a diagonal matrix $D = \diag(d_1,\ldots,d_K)$, with positive on-diagonal
elements $d_i$, we associate an inner product and a norm as follows:
\[
\forall \, \b x, \b y \in \R^K, \qquad \langle \b x, \, \b y \rangle_H = \b x^\top H \b y
\qquad \mbox{and} \qquad
\norm{\b x}_H = \sqrt{\b x^\top H \b x} \,.
\]
We denote by $\pi_D$ the projection on $\R_-^K$ under the norm $\norm{\,\cdot\,}_D$.
It turns out that this projection is independent of the considered matrix $D$ satisfying the
constraints described above: it equals
\[
\forall \b x \in \R^K, \qquad \pi_{D}(\b x) = \b x - \b x_+\,,
\]
where we recall that $\b x_+$ denotes the vector whose components are the nonnegative parts
of the components of $\b x$.
This entails that for all $\b x,\b y \in \R^K$
\begin{equation}
\label{eq:proj}
	\norm{(\b x+ \b y)_+}_{D} = \norm{\b x+\b y- \pi_{D}(\b x+\b y)}_{D}
		\leq \norm{\b x+\b y- \pi_{D}(\b x)}_{D}^2 \\
		= \norm{\b x_+ + \b y}_{D}^2 \,.
\end{equation}
\medskip

Now, we consider, for each instance $t \geq 1$, the diagonal matrix $D_t = \diag(\eta_{1,t},\dots,\eta_{K,t})$,
with positive elements on the diagonal. As all sequences $(\eta_{k,t})_{t \geq 0}$ are non-increasing for
a fixed $k$, we have, for all $t \geq 1$, that
\begin{equation}
\label{eq:normdecreasing}
\forall \b x \in \R^K,  \qquad \norm{\b x}_{D_t} \leq \norm{\b x}_{D_{t-1}}\,.
\end{equation}
This entails that
\begin{equation} \label{eq:Rt}
	\norm{\left(\b R_t\right)_+}_{D_t} \leq \norm{\left(\b R_t\right)_+}_{D_{t-1}}
		= \norm{\left(\b R_{t-1} + \b r_t\right)_+}_{D_{t-1}}
\leq \norm{\left(\b R_{t-1}\right)_+ + \b r_t}_{D_{t-1}} \,,
\end{equation}
where we denoted by $\b r_t$ the vector $(r_{k,t})_{1 \leq k \leq K}$ of the instantaneous
regrets and where we applied~\eqref{eq:proj}.
Taking squares and developing the squared norm, we get
\begin{equation}
\label{eq:regret1}
\norm{\left(\b R_t\right)_+}_{D_t}^2
\leq \norm{\left(\b R_{t-1}\right)_+}_{D_{t-1}}^2 + \norm{\b r_t}_{D_{t-1}}^2 \\
		+ 2\,\b r_t^\top D_{t-1} \left(\b R_{t-1}\right)_+.
\end{equation}
But the inner product equals
\begin{equation*}
2 \, \b r_t^\top D_{t-1} \left(\b R_{t-1}\right)_+ =
 2 \sum_{k=1}^K \eta_{k,t-1} \left(R_{k,t-1}\right)_+r_{k,t}\\
	 =  2 \, \b \eta_{t-1}^\top \left(\b R_{t-1}\right)_+ \underbrace{\sum_{k=1}^K p_{k,t} r_{k,t}}_{=0} = 0\,,
\end{equation*}
where the last but one equality follows from step~1 of the algorithm.

Hence \eqref{eq:regret1} entails $\norm{\left(\b R_t\right)_+}_{D_t}^2
- \norm{\left(\b R_{t-1}\right)_+}_{D_{t-1}}^2 \leq \norm{\b r_t}_{D_{t-1}}^2$, which, summing over all
rounds $t \geq 1$, leads to
\begin{multline}
\label{eq:regret2}
\norm{\left(\b R_T\right)_+}_{D_T}^2  - \cancel{\norm{\left(\b R_0\right)_+}_{D_0}^2}  \leq
\sum_{t=1}^T \norm{\b r_t}_{D_{t-1}}^2
= \sum_{t=1}^T \sum_{k=1}^K \eta_{k,t-1}r_{k,t}^2 \\
= \sum_{k=1}^K \sum_{t=1}^T \frac{r_{k,t}^2}{1+\sum_{s=1}^{t-1} r_{k,s}^2}
\leq K \bigl(  1+  \ln (1+ T)\bigr)\,,
\end{multline}
where the last but one inequality follows from substituting the value of $\eta_{k,t-1}$
and the last inequality was proved in~\eqref{eq:intermpoly}.
Finally, \eqref{eq:regret2} implies that, for any expert $k=1,\dots,K$,
\begin{equation*}
	\eta_{k,T} \left(R_{k,T}\right)_+^2
		\leq \norm{\left(\b R_T\right)_+}_{D_T}^2
		\leq K \bigl(  1+  \ln (1+ T)\bigr)\,,
\end{equation*}
so that
\[
	R_{k,T} \leq \sqrt{K\bigl(1+\ln \!\left(1+ T\right)\bigr) \, \eta_{k,T}^{-1}}\,.
\]
The proof is concluded by substituting the value of $\eta_{k,T}$.

\subsection{Proof of Theorem~\ref{th:Bern} (variation on the Bernstein--Freedman inequality)}
\label{sec:Bern}

Let $\phi : \R \to \R$ and $\varphi : \R \to \R$ be defined by
$\phi(\lambda) = \e^\lambda - \lambda - 1$ on the one hand, $\varphi(0)
= 1/2$ and $\varphi(\lambda) = \phi(\lambda)/\lambda^2$ on the other
hand. The following lemma is due to \citet[Lemmas 1.3a and 3.1]{Freed}.
Note that we are only proving a one-sided inequality and do not require
the lower bound on $X$ imposed in the mentioned reference.

\begin{lemma}[{\citealp{Freed}}]
\label{lm:lmber}
The function $\varphi$ is increasing.
As a consequence, for all bounded random variables $X \leq 1$ a.s., for all
$\sigma$--algebras $\cF$ such that $\E[X\,|\,\cF] = 0$ a.s., and for all
nonnegative random
variables $\Lambda \geq 0$ that are $\cF$--measurable,
\[
\E\big[ \exp(\Lambda X) \,\big| \, \cF \big] \leq \exp\big( \phi(\Lambda) \, V \big) \quad \mbox{\rm a.s.} \quad
\qquad \mbox{where} \qquad
V = \E\big[ X^2 \,\big| \, \cF \big] = \Var(X\,|\,\cF)\,.
\]
\end{lemma}

\begin{proof}
That $\varphi$ is increasing follows from a function study. Using this, we get $\varphi(\Lambda X) \leq \varphi(\Lambda)$,
which can be rewritten as
\[
\e^{\Lambda X} - \Lambda X - 1 \leq \phi(\Lambda) \, X^2\,.
\]
By integrating both sides with respect to $\E[\,\cdot\,\,|\,\cF]$ and by using that $\Lambda$ is
$\cF$--measurable and that $\E[X\,|\,\cF] = 0$ a.s., we get
\[
\E\big[ \exp(\Lambda X) \,\big| \, \cF \big]
\leq 1 + \phi(\Lambda) \, V \qquad \mbox{a.s.}
\]
The proof is concluded by the inequality $1+u \leq \e^u$, valid for all $u \in \R$.
\end{proof}

\begin{proof}\textbf{[of Theorem~\ref{th:Bern}]}
We fix $x > 0$.
The analysis relies on a non-increasing sequence
of random variables $1 \geq \Lambda_1 \geq \Lambda_2 \geq \ldots > 0$ such that each $\Lambda_t$
is $\cF_{t-1}$--measurable. More precisely, we pick
\[
\Lambda_t = \min\,\left\{1,\,\, \sqrt{\frac{x}{1 + \sum_{s=1}^{t-1} V_s}} \right\}
\]
and choose, by convention, $\Lambda_0 = 1$.
We define, for all $t \geq 1$,
\[
H_t = \exp \!\left( \Lambda_t \sum_{s=1}^t \left( X_s - \frac{\phi(\Lambda_s)}{\Lambda_s} V_s \right) \right)
\]
and $H_0 = 1$.
Below we will apply Markov's inequality to $H_t$ and we therefore need to
bound $\E[H_t]$. By Lemma~\ref{lm:lmber},
\[
\E\Bigl[ \exp \bigl( \Lambda_t X_t - \phi(\Lambda_t) V_t \bigr) \,\Big|\, \cF_{t-1} \Bigr] \leq 1 \qquad \mbox{a.s.},
\]
so that for all $t \geq 1$,
\begin{multline*}
\E[H_t] =
\E\bigl[ \E[H_t\,|\,\cF_{t-1}] \bigr] \\
= \E\!\!\left[ \exp \!\left( \Lambda_t \sum_{s=1}^{t-1} \left( X_s - \frac{\phi(\Lambda_s)}{\Lambda_s} V_s \right) \right)
\,\, \E\Bigl[ \exp \bigl( \Lambda_t X_t - \phi(\Lambda_t) V_t \bigr) \,\Big|\, \cF_{t-1} \Bigr] \right]
\leq \E\Bigl[ H_{t-1}^{\Lambda_t/\Lambda_{t-1}} \Bigr]\,.
\end{multline*}
Applying Lemma~\ref{lm:ineq} with $\alpha = \Lambda_{t-1}/\Lambda_t$, this can be further bounded as
\[
\E[H_t] \leq \E\Bigl[ H_{t-1}^{\Lambda_{t}/\Lambda_{t-1}} \Bigr] \leq
\E[H_{t-1}] + \frac{1}{\e} \, \E\!\!\left[\frac{\Lambda_{t-1}}{\Lambda_{t}}-1\right].
\]
Proceeding by induction and given that $H_0 = 1$, we get, for all $T \geq 1$,
\[
\E[H_T] \leq 1 + \frac{1}{\e} \sum_{t=1}^T \E\!\left[\frac{\Lambda_{t-1}}{\Lambda_{t}}-1\right].
\]
The same argument and calculations as in~\eqref{eq:sumint1} and~\eqref{eq:intermpoly} finally show that
\[
\E[H_T] \leq \underbrace{1 + \frac{1}{2\e} \,\E\!\!\left[1+\ln\left( 1 + \sum_{t=1}^T V_t \right)\right]}_{\leq \gamma};
\]
that the left-hand side is less than $\gamma$ follows from Jensen's inequality for the logarithm.
An application of Markov's inequality entails that
\[
\P \! \left\{\sum_{t=1}^T X_t \geq \frac{x}{\Lambda_T} + \sum_{t=1}^T \frac{\phi(\Lambda_t)}{\Lambda_t} V_t \right\}
= \P \bigl\{ \ln H_T \geq x \bigr\}
= \P \bigl\{ H_T \geq \e^{x} \bigr\} \leq \E[H_T] \, \e^{-x}\,.
\]
To conclude the proof, it thus suffices to take $x$ such that
\[
\E[H_T] \, \e^{-x} \leq \delta\,, \qquad \mbox{e.g.,} \qquad
x = \ln \frac{\gamma}{\delta}
\]
and to show that
\begin{equation}
\label{eq:bernconclu}
\frac{x}{\Lambda_T} + \sum_{t=1}^T \frac{\phi(\Lambda_t)}{\Lambda_t} V_t
\leq 3 \sqrt{x} \sqrt{1 + \sum_{t=1}^{T} V_t} + x \,,
\end{equation}
which we do next.

Because $\Lambda_t \leq 1$ and $\varphi$ is increasing, we have $\varphi(\Lambda_t) = \phi(\Lambda_t)/\Lambda_t^2
\leq \varphi(1) = \e-2 \leq 1$. Therefore,
\[
\sum_{t=1}^T \frac{\phi(\Lambda_t)}{\Lambda_t} V_t \leq \sum_{t=1}^T \Lambda_t V_t
\leq \sum_{t=1}^T \sqrt{\frac{x}{1 + \sum_{s=1}^{t-1} V_s}} \,\, V_t
\leq 2 \sqrt{x} \sqrt{1+\sum_{t=1}^T V_t}\,,
\]
where we used for the second inequality the definition of $\Lambda_t$ as a minimum and
applied the same argument as in~\eqref{eq:lmRsqrt} for the third one.
It only remains to bound $x/\Lambda_T$, for which we use the upper bound (again, following
from the definition of $\Lambda_T$ as a minimum)
\[
\frac{x}{\Lambda_T} \leq x + \sqrt{x} \sqrt{1 + \sum_{t=1}^{T-1} V_t}
\leq x + \sqrt{x} \sqrt{1 + \sum_{t=1}^{T} V_t}\,.
\]
Putting things together, we proved~\eqref{eq:bernconclu}, which concludes this proof.
\end{proof}

\section{Additional material for Section~\ref{sec:appliconf}}

\subsection{The gradient trick --- how to deal with convex losses via a reduction to the linear case}
\label{sec:gradient}

\citet{FreundSchapireSingerWarmuth1997} consider the case
of convex aggregation in the context of sleeping experts and design several strategies,
each specific to a convex loss function. Devaine \emph{et al.}\
explain in Section~2.2 of \citet{DevaineGaillardGoudeStoltz2013} how to reduce the problem
of convex aggregation to linear losses, via the standard gradient trick (see, e.g., Section~2.5 of \citealp{CesaBianchiLugosi2006}), and could exhibit a unified analysis
of all the strategies of~\citet{FreundSchapireSingerWarmuth1997}.

We briefly recall this reduction here and note that it also holds for the generalization
from sleeping experts to experts that report their confidences.

\paragraph{Setting and notation (see~\citealt{FreundSchapireSingerWarmuth1997}).}
Suppose the experts predict by choosing an element $x_{k,t}$ from a convex set $\mathcal{X} \subseteq \R^d$ of possible
predictions, and that their losses at round $t$ are determined by a convex and differentiable
function $f_t$, such that $\loss_{k,t} = f_t(x_{k,t})$.
At each step, the forecaster chooses a weight vector $\bp_t$ over $\cA_t$
and aggregates the expert forecasts as
\[
\hat{x}_t = \sum_{k \in \cA_t} p_{k,t} x_{k,t}\,,
\]
with resulting loss $\algloss_t = f_t \bigl( \hat{x}_t \bigr)$.

Instead of competing with the best expert, we may wish to compete with the
best fixed convex combination of experts in the following way.
At round $t$, a weight vector $\bq$ with nonnegative components that
sum to $1$ aggregates the forecasts according to
\begin{equation*}
x_{\bq,t} = \sum_{k \in \cA_t} \frac{q_k I_{k,t}}{Q_t(\bq)} x_{k,t}
\quad \mbox{where} \quad Q_t(\bq) = \sum_{k' \in \cA_t} q_{k'} I_{k',t}\,;
\end{equation*}
the resulting loss equals $f_t(x_{\bq,t})$.

The regret with respect to a given $\bq$ is then defined as
\begin{equation*}
\cregret_{\bq,T} = \sum_{t=1}^T Q_t(\bq) \big(\algloss_t - f_t(x_{\bq,t}) \big),
\end{equation*}
which reduces to the confidence regret of Section~\ref{sec:descr} if $\bq$ is a point-mass.

\paragraph{The reduction to linear losses.}
We may now reduce this problem to case of linear losses considered in
Sections~\ref{sec:introduction} and~\ref{sec:descr}.
We do so by resorting to the so-called gradient trick.
We denote by $\nabla f_t$ the gradient of $f_t$ and
introduce pseudo-losses $\loss'_{k,t} = \nabla f_t(\hat{x}_t)^\top x_{k,t}$
for all experts $k$. We denote by $\b\ell'_t$ the vector of the pseudo-losses.
Because of the convexity inequality
\[
f_t(y) \geq f_t(x) + \nabla f_t(x)^\top(y-x)
\qquad\forall \, x,y \in \mathcal{X},
\]
we have
\begin{align*}
\max_{\bq} \cregret_{\bq,T}
	= &\   \max_{\bq} \sum_{t=1}^T Q_t(\bq) \Big( f_t \bigl( \hat{x}_t \bigr) - f_t(x_{\bq,t}) \Big) \\
	\leq &\   \max_{\bq} \sum_{t=1}^T Q_t(\bq) \Big( \nabla f_t\bigl(\hat{x}_t\bigr)^\top \bigl( \hat{x}_t - x_{\bq,t} \bigr) \Big) \\
	= &  \ \max_{\bq} \sum_{t=1}^T Q_t(\bq) \left( \bp_t^\top {\b\ell}'_t - \sum_{k \in \cA_t} \frac{q_k I_{k,t}}{Q_t(\bq)} \ell'_{k,t} \right)\!.
\end{align*}
Substituting the definition of $Q_t(\bq)$, we get that $
\max_{\bq} \cregret_{\bq,T}$ is upper-bounded by
\begin{equation*}
	\max_{\bq} \sum_{t=1}^T \left( \sum_{k' \in \cA_t} q_{k'} I_{k',t} \bp_t^\top {\b\ell}'_t
	- \sum_{k \in \cA_t} q_k I_{k,t} \ell'_{k,t} \right)
	 =  \max_{\bq}\sum_{k=1}^K q_k \, \underbrace{\sum_{t=1}^T I_{k,t} \bigl( \bp_t^\top {\b\ell}'_t - \ell'_{k,t} \bigr)}_{\cregret_{k,T}}
	=  \max_{k} \cregret_{k,T}\,,
\end{equation*}
where the first equality is because $I_{k,t} = 0$ for $k
\not\in \cA_t$,
and the last equality follows by linearity of the expression in $\bq$.

Therefore, any regret bound for the linear prediction
setting with losses $\loss'_{k,t}$ implies a bound for competing
with the best convex combination of expert predictions in the original
convex setting with losses $\loss_{k,t}$.

\subsection{Hedge with multiple learning rates for experts that report
their confidences}
\label{sec:BM}

In this section, we discuss another algorithm with multiple learning rates, which
was proposed by \citet{BlumMansour2007}. We slightly adjust its
presentation so that it fits the setting of this paper: Blum and Mansour
always consider all combinations of $K$ experts and $M$ confidences
$\mathcal{M} = \{I_{1,t},\ldots,I_{M,t}\}$, which they refer to as
``time selection functions.'' These enter as $\sqrt{\ln (KM)}$ in their
Theorem~16. To recover their setting, we can consider $M$ copies of each
expert, one for each ``time selection function'', so that our effective
number of experts becomes $KM$ and we also obtain a $\sqrt{\ln (KM)}$
factor in our bounds. Converse, to couple time selection functions and
experts, like we do, Blum and Mansour (see their Section~6) simply take
$\mathcal{M} = \{I_{1,t},\ldots,I_{K,t}\}$, so that $M = K$ and hence
they obtain $\sqrt{\ln (KM)} = \sqrt{2\ln K}$, which is the same as our
$\sqrt{\ln K}$ up to a factor $\sqrt{2}$. Thus the two settings are
essentially equivalent.
\begin{algorithm}[thb]
\caption{Hedge with multiple learning rates for experts reporting confidences  (MLC-Hedge)}
\label{algo:hedge}
	\smallskip
    	\emph{Parameters}: a vector $\b\eta = (\eta_1,\ldots,\eta_K)$ of learning rates \\
	\emph{Initialization}: a vector $\b w_0 =
        (w_{1,0},\ldots,w_{K,0})$ of nonnegative weights that sum to $1$
	
	\smallskip
	\emph{For} each round $t = 1,\,2,\,\ldots$ \\
	\indent 1. form the mixture $\bp_t$ defined by
$\displaystyle{p_{k,t} = \frac{{I_{k,t} \bigl( 1 -\e^{-\eta_{k}} \bigr) w_{k,t-1} }}{\sum_{k'=1}^K
I_{k',t} \bigl( 1 -\e^{-\eta_{k'}} \bigr) w_{k',t-1}}}$ \vspace{.2cm} \\
	\indent 2. observe the loss vector $\b \ell_t$ and incur loss $\algloss_t = \b p_t^\top \b \ell_t$\\
	\indent 3. for each expert $k$ perform the update
		 $\displaystyle{w_{k,t} = w_{k,t-1} \exp \Bigl(
\eta_k \, I_{k,t} \bigl( \e^{-\eta_k} \algloss_t - \ell_{k,t} \bigr) \Bigr)}$ \\
\end{algorithm}

\begin{theorem}[Adapted from {\citealp{BlumMansour2007}}]
\label{th:BM}
For all $K$-tuples $\b\eta$ of positive learning rates in $[0,1]^K$, for all sequences of loss vectors $\b \ell_t \in [0,1]^K$
and of confidences $(I_{1,t},\ldots,I_{K,t}) \in [0,1]^K$, the confidence regret of Algorithm~\ref{algo:hedge} is bounded as
follows: for all experts $k \in \{1,\ldots,K\}$,
\begin{equation}
\label{eqn:MLCHedgeRegret}
R_{k,t}^c
= \sum_{t=1}^T I_{k,t} \bigl( \algloss_t - \ell_{k,t} \bigr)
 \leq \frac{\ln(1/w_{k,0})}{\eta_k} + (\e-1) \eta_k \sum_{t=1}^T I_{k,t} \ell_{k,t}
  +  (\e-1) \ln(1/w_{k,0})\,.
\end{equation}
\end{theorem}

Optimizing \eqref{eqn:MLCHedgeRegret} with respect to $\eta_k$,
for all $k \in \{1,\ldots,K\}$  we obtain
\begin{equation*}
R_{k,t}^c
	\leq \quad 2 \sqrt{(\e-1) \sum_{t=1}^T I_{k,t} \ell_{k,t} \ln(1/w_{k,0})}
	  + (\e-1) \ln(1/w_{k,0})\,,
\end{equation*}
as indicated in~\eqref{eq:BlumMansour2007}.

\begin{remark}
\label{rm:BM}
Although, in practice, we cannot optimize \eqref{eqn:MLCHedgeRegret} with respect to $\eta_k$,
it is possible to tune the parameters $\eta_{k,t}$ of MLC-Hedge sequentially using a
similar approach as in the proof of
Theorem~\ref{th:lmprod-adapt}, at the same small $O(\ln \ln T)$ cost. (We believe that there is some cost
here for this tuning; the bound stated in Section 6 of~\citet{BlumMansour2007} only considers the
case of an optimization in hindsight and alludes to the possibility of some online tuning, not working out the details.)
\end{remark}

The analysis of
MLC-Hedge suggests that its bound
can probably not be obtained in a two-step procedure, by first
exhibiting a bound in the standard setting for some
ML-Hedge algorithm and then applying the generic reduction from
Section~\ref{sec:genred} to get an algorithm
suited for experts that report their confidences.
Thus, the approach taken in the main body of this paper seems more general. \\

\begin{proof}\textbf{[of Theorem~\ref{th:BM}]}
As in the proof of Theorem~\ref{th:lmprod-plain}, we upper and lower bound $\ln W_T$.
For all $k$, the lower bound $W_T \geq w_{k,T}$ together with the fact that
\[
w_{k,T} = w_{k,0} \, \exp \!\left( \eta_k \sum_{t=1}^T I_{k,t} \bigl( \e^{-\eta_k} \algloss_t - \ell_{k,t} \bigr) \right),
\]
yields
\[
\sum_{t=1}^T I_{k,t} \bigl( \e^{-\eta_k} \algloss_t - \ell_{k,t} \bigr) \leq \frac{\ln W_T + \ln(1/w_{k,0})}{\eta_k}\,,
\]
which entails
\begin{equation}
\label{eq:blma}
\sum_{t=1}^T  I_{k,t}   \algloss_t \leq \left(\sum_{t=1}^T I_{k,t} \ell_{k,t}
+ \frac{\ln W_T + \ln(1/w_{k,0})}{\eta_k} \right) \e^{\eta_k}\,.
\end{equation}

We now upper-bound $W_T$ by $W_0 = 1$. To do so, we show that $W_{t+1} \leq W_t$ for all $t \geq 0$.
By the weight update (step 3 of the algorithm), $W_t = \sum_{k=1}^K w_{k,t} $ equals
\begin{equation}\label{eq:WBM}
\sum_{k=1}^K w_{k,t-1} \exp \Bigl( \eta_k \, I_{k,t} \bigl( \e^{-\eta_k} \algloss_t - \ell_{k,t} \bigr) \Bigr) \\
 = \sum_{k=1}^K w_{k,t-1} \, \exp \bigl( - \eta_k I_{k,t} \ell_{k,t} \bigr) \, \exp \bigl( \eta_k \e^{-\eta_k} I_{k,t} \algloss_t \bigr).
\end{equation}
For all $\eta \in \R$, the function $x \in [0,1] \mapsto \e^{\eta x}$ is convex, and therefore,
\[
\e^{\eta x} \leq (1-x) \e^{0} + x \, \e^\eta = 1 -\bigl( 1-\e^\eta
\bigr)x\,.
\]
In particular for all $\eta > 0$ and for all $x \in [0,1]$
\[
\e^{-\eta x} \leq 1 - \bigl( 1 - \e^{-\eta} \bigr) x
\]
and
\begin{equation}
\label{eq:expeta}
\e^{\eta x} \leq 1 + \bigl( 1 - \e^{-\eta} \bigr) \e^\eta x\,.
\end{equation}
Bounding~\eqref{eq:WBM} further with the two inequalities stated above, we get
\begin{align*}
W_t 	\leq & \quad \sum_{k=1}^K w_{k,t-1} \,
\Bigl( 1 - \bigl( 1 - \e^{-\eta_k} \bigr) I_{k,t} \ell_{k,t} \Bigr)
 	\Bigl( 1 + \bigl( 1 - \e^{-\eta_k} \bigr) {\cancel{\e^{\eta_k} \, \e^{-\eta_k}}} I_{k,t} \algloss_t \Bigr)  \\
 	\leq & \quad \sum_{k=1}^K w_{k,t-1} \Bigl( 1 + \bigl( 1 - \e^{-\eta_k} \bigr) I_{k,t} \bigl( \algloss_t - \ell_{k,t} \bigr) \Bigr) \\
 	=   & \quad  W_{t-1} +  \sum_{k=1}^K \underbrace{w_{k,t-1} \bigl( 1 - \e^{-\eta_k} \bigr) I_{k,t}}_{= Z_t \, p_{k,t}}
\bigl( \algloss_t - \ell_{k,t} \bigr) \\
 	=   & \quad  W_{t-1} +  Z_t \Biggr( \underbrace{\algloss_t - \sum_{k=1}^K p_{k,t} \ell_{k,t}}_{ = 0} \Biggr)
 	=      W_{t-1}\,,
\end{align*}
where $Z_t =  \sum_{k'=1}^K w_{k',t-1} \bigl( 1 - \e^{-\eta_{k'}} \bigr) I_{k',t}$ and the first
equality is by the definition of $\b p_t$ (step~1 of the algorithm). This concludes the induction.

We then get from~\eqref{eq:blma}
\[
\sum_{t=1}^T  I_{k,t}   \algloss_t
\leq \left(\sum_{t=1}^T I_{k,t} \ell_{k,t}  + \frac{\ln(1/w_{k,0})}{\eta_k} \right)  \e^{\eta_k} \,.
\]
The claim of the theorem follows by
the upper bound $\e^{\eta_k} \leq 1 +(\e-1) \eta_k$ for $\eta_k \in [0,1]$, which is a special case of~\eqref{eq:expeta}.
\end{proof}

\end{document}